\documentclass[runningheads, envcountsame, a4paper]{llncs}
\usepackage{amssymb}
\usepackage{dirtytalk}
\usepackage[pdftex]{graphicx}
\usepackage{epstopdf}
\usepackage[hyphens]{url}
\usepackage{booktabs}
\usepackage[short]{optidef}
\usepackage[ruled,vlined]{algorithm2e}
\usepackage{float}
\usepackage{fancyhdr}
\floatstyle{plaintop}
\restylefloat{table}
\usepackage{microtype}     
\usepackage[misc]{ifsym}
\def\R{{\mathbb{R}}}
\def\C{{\mathcal{C}}}
\def\Cdyn{{\mathcal{C}^\text{dyn}}}
\def\Clag{{\mathcal{C}^\text{lag}}}
\def\dm{{\mathrm{d}}}
\def\Lo{{\mathcal{L}}}
\def\bdot{{\boldsymbol{\cdot}}}
\def\path{{\phi}}
\def\x{{\textnormal{x}}}
\def\am{{\mathrm{\alpha}}}
\def\bm{{\mathrm{\beta}}}
\newcommand\blfootnote[1]{%
  \begingroup
  \renewcommand\thefootnote{}\footnote{#1}%
  \addtocounter{footnote}{-1}%
  \endgroup
}

\begin{document}

\title{A Principle of Least Action for the Training of Neural Networks}
\titlerunning{A Principle of Least Action for the Training of Neural Networks}

\author{Skander Karkar*\inst{1} [\Letter] \and Ibrahim Ayed*\inst{2}\and Emmanuel de Bézenac* \inst{2} \and Patrick Gallinari \inst{1,2}}
\authorrunning{S. Karkar* et al.}
\institute{
Criteo AI Lab, Criteo, Paris, France \\ \email{as.karkar@criteo.com} \and 
LIP6, Sorbonne Université, Paris, France \email{$\{$ibrahim.ayed,emmanuel.de-bezenac,patrick.gallinari$\}$@lip6.fr}
}

\maketitle
\toctitle{A Principle of Least Action for the Training of Neural Networks}
\tocauthor{Skander~Karkar* \and Ibrahim~Ayed* \and Emmanuel~de~Bézenac* \and Patrick~Gallinari}
\setcounter{footnote}{0}

\begin{abstract}
\blfootnote{* denotes equal contribution.}
Neural networks have been achieving high generalization performance on many tasks despite being highly over-parameterized. Since classical statistical learning theory struggles to explain this behaviour, much effort has recently been focused on uncovering the mechanisms behind it, in the hope of developing a more adequate theoretical framework and having a better control over the trained models. In this work, we adopt an alternative perspective, viewing the neural network as a dynamical system displacing input particles over time. We conduct a series of experiments and, by analyzing the network's behaviour through its displacements, we show the presence of a low kinetic energy bias in the transport map of the network, and link this bias with generalization performance. From this observation, we reformulate the learning problem as follows: find neural networks that solve the task while transporting the data as efficiently as possible. This offers a novel formulation of the learning problem which allows us to provide regularity results for the solution network, based on Optimal Transport theory. From a practical viewpoint, this allows us to propose a new learning algorithm, which automatically adapts to the complexity of the task, and leads to networks with a high generalization ability even in low data regimes.
\end{abstract}

\section{Introduction}

Deep neural networks (DNNs) have repeatedly shown their ability to solve a wide range of challenging tasks, while often having many more parameters than there are training samples. Such a performance of over-parametrized models is counter-intuitive. They seem to adapt their complexity to the given task, systematically achieving a low training error without suffering from over-fitting as could be expected~\cite{tradeoff,double,rethinking_generalization}. This is in contradiction with the classical statistical practice of selecting a class of functions complex enough to represent the coherent patterns in the data, and simple enough to avoid spurious correlations~\cite{kernel_learning,statistical_learning}. Although this behavior has sparked much recent work towards explaining neural networks' success~\cite{simple_functions,ntk,sensitivity,spectral_bias}, it still remains poorly understood. Among the factors to consider are the implicit biases present in the choices made for the parametrization, the architecture, the parameter initialization and the optimization algorithm, and that contribute all to this success. Our aim in this work is to uncover some of these hidden biases and highlight their link with generalization performance through the lens of dynamical systems.

We will focus on residual networks~(ResNets)~\cite{he2016b,he2016a}, now ubiquitous in applications. This family of models has made it possible to learn very complex non-linear functions by improving the trainability of very deep networks, and has thus improved generalization. Links have been derived between these networks and dynamical systems: a ResNet can be seen as a forward Euler scheme discretization of an associated ordinary differential equation (ODE)~\cite{weinan}:

\begin{equation}
    \x_{k+1} = \x_k + v_k(\x_k)
    \; \longleftrightarrow \;
    \partial_t \x_t = v_t(\x_t)
    \label{eq:resnet}    
\end{equation}

This link has yielded many exciting results, \textit{e.g.} new architectures~\cite{lm} and reversible  networks~\cite{reversible}. Here, we make use of this analogy and analyze the behavior of residual networks by studying their associated differential flows. Adopting this dynamical point of view allows us to leverage the theories and mathematical tools developed to study, approximate and apply differential equations.

More specifically, we conduct experiments to observe how neural networks displace their inputs--seen as particles--through time. We measure a strong empirical correlation between good test performance and neural networks with low kinetic energy along their transport flow. From this, we reformulate the training problem as follows: retrieve the network which solves the task using the principle of least action, \textit{i.e.} expending as little kinetic energy as possible. This problem, in its probabilistic formulation, is tightly linked with and inspired by the well-known problem of finding an optimal transportation map \cite{Santambrogio}. This yields new insights into neural networks' generalization capabilities, and provides a novel algorithm that automatically adapts to the complexity of the data and robustly improves the network's performance, including in low data regimes, without slowing down the training. To summarize, our contributions are the following:
\begin{itemize}
    \item Through the dynamic viewpoint, we highlight the \textit{low-energy bias} of ResNets.
    \item We formulate a Least Action Principle for the training of Neural Networks.
    \item We prove existence and regularity results for networks with minimal energy.
    \item We provide an algorithm for retrieving minimal energy networks compatible with different architectures, which leads to better generalization performance on different classification tasks, without complexifying the architecture.
\end{itemize}

We introduce in Section~\ref{sec:back} some background on Optimal Transport~(OT) and highlight the link between the dynamical formulation of OT and ResNets. We describe in Section~\ref{sec:gen_sett} the general setting of our analysis. Section ~\ref{sec:emp_ot_an} provides empirical evidence illustrating our point. The formal framework of networks trained with minimized energy and a practical algorithm are described in Section~\ref{sec:leap}. Experiments on standard classification tasks are provided in Section~\ref{sec:expes}. The code is available online at \texttt{github.com/skander-karkar/LAP}.

\section{\label{sec:back}Background}

This section outlines the main elements of the formalism and reasoning of our work. Supplementary Material A gives more details about Optimal Transport.

\subsection{\label{subsec:back_ot}Optimal Transport}

The principle of least action is central to many fields in physics, mathematics and economics. It is found in classical and relativistic mechanics, thermodynamics, quantum mechanics~\cite{feynman,thermo,gray}, etc.. It broadly states that the dynamical trajectory of a system between an initial and final configuration is one that makes a certain action associated with the system locally stationary~\cite{gray}. One mathematical theory which can be associated with this general idea is the theory of Optimal Transport which was initially introduced as a way of finding a transportation map minimizing the cost of displacing mass from one configuration to another~\cite{Santambrogio}.

Formally, let $\am$ and $\bm$ be absolutely continuous distributions compactly supported in $\R^d$, and ${c: \R^d \times \R^d \to \R}$ a cost function. Consider a transportation map ${T: \R^d \to \R^d}$ that satisfies ${T_\sharp \alpha =\beta}$, \textit{i.e.} that pushes\footnote{$T_\sharp \alpha$ is the  \textit{push-forward measure}: $T_\sharp \alpha (B) = \alpha(T^{-1}(B))$ for any measurable set $B$.} $\am$ to $\bm$. The total cost of the transportation then depends on all the individual contributions of costs of transporting (infinitesimal) mass from each point $x$ to $T(x)$, and finding the optimal transportation map amounts to solving:
\begin{mini}
    {T}{ \mathcal{C}^{\text{stat}}(T) = \int_{\mathbb{R}^d} c(x, T(x))\dm \alpha(x)}
    {}{}
    \addConstraint{T_\sharp \alpha }{=\beta}
    \label{eq:monge}
\end{mini}
A standard choice for $c$ is the $p$-th power of a norm of $\R^d$, \textit{i.e.} $c(x, y) = \|x-y\|^p$, but other costs can be used, defining different variants of the problem. This cost induces, through the $p$-th root of the minimal value of~\eqref{eq:monge}, a distance $W_p$ between any two distributions $\alpha$ and $\beta$ of finite $p-$th moment, called the $p$-Wasserstein distance~\cite{cot}.

In \cite{brenier}, the link between Optimal Transport and the principle of least action was made by showing that the static transportation can equivalently be viewed as a dynamical one that minimizes an action as it gradually displaces particles of mass in time. In other words, instead of directly pushing samples of $\am$ to $\bm$ in $\mathbb{R}^d$ using $T$, we can displace mass from $\am$ according to a continuous flow with velocity ${v_t: \R^d \to \R^d}$. This implies that the density $\mu_t$ at time $t$ satisfies the \textit{continuity equation} ${\partial_t \mu_t + \nabla \cdot (\mu_t v_t) =0}$, assuming that initial and final conditions are given respectively by $\mu_0=\alpha$ and $\mu_1 = \bm$. In this case, the optimal displacement is the one that minimizes the action $\|v_t\|^p_{L^p(\mu_t)}$:
\begin{mini}
    {v}{ \Cdyn(v) = \int_0^1 \|v_t\|^p_{L^p(\mu_t)}\dm t}
    {}{}
    \addConstraint{\partial_t \mu_t + \nabla \cdot (\mu_t v_t)}{= 0, \mu_{0}= \am, \mu_{1}= \bm}
    \label{eq:dyn_ot}
\end{mini}
where $\|v_t\|^p_{L^p(\mu_t)}=\int_{\R^d}\| v_t(x)\|^p \dm \mu_t(x)$ for costs $c(x, y) = \|x-y\|^p$ with $p > 1$. In this case, minimizers exist and the two transport costs are the same, \textit{i.e.} $\mathcal{C}^{\text{stat}}(T) = \Cdyn(v)$ at the optimums. For $p=2$ and the Euclidean norm, the dynamical cost $\Cdyn(v)$ corresponds to the \textit{kinetic energy}.

\subsection{\label{subsec:back_otres}Link with Residual Networks}

The dynamical formulation in~\eqref{eq:dyn_ot} explicitly describes the evolution in time of the density $\mu_t$, starting from an input distribution $\am$. In this form, the link between deep residual networks and dynamical Optimal Transport is not clear. However, it is possible to adopt an alternate viewpoint which helps make it immediate. Instead of explicitly describing the density's evolution, we describe the paths $\path^x\colon [0, 1] \to \R^d\!, \: t \mapsto \path^x_t$ taken by particles from $\am$ at position $x$, when displaced along the flow $v$. The continuity equations can then equivalently be written as:
\begin{equation}
    \label{eq:continuity}
    \partial_t\path_t^x = v_t(\path_t^x)
\end{equation}

See chapters 4 and 5 of \cite{Santambrogio} for details. We can now note the resemblance between the residual network~\eqref{eq:resnet} and equation~\eqref{eq:continuity}. Rewriting the conditions as necessary, the dynamical formulation~\eqref{eq:dyn_ot} can equivalently be represented by:
\begin{mini}
    {v}{ \Clag(v) = \int_0^1 \|v_t\|_{L^p((\path^\bdot_t)_\sharp\am)}^p\, \dm t}
    {}{}
    \addConstraint{\partial_t\path_t^x}{ = v_t(\path_t^x)}
    \addConstraint{\path^\bdot_0}{= \text{id}}
    \addConstraint{(\path^\bdot_1)_\sharp\mathcal{\am} }{= \bm}
    \label{eq:ot_lag}
\end{mini}
where $\path^\bdot_t: x \in \R^d \mapsto \path^x_t \in \R^d$ corresponds to the transport map induced by the flow, up until time $t$. As both formulations are equivalent, we have that for any flow $v$, $\Clag(v) = \Cdyn(v)$. Moreover, optimal transportation plans in the static~\eqref{eq:monge} and dynamical~\eqref{eq:ot_lag} cases coincide: if $T$ and $\path^\bdot_t$, are respectively solutions to~\eqref{eq:monge} and~\eqref{eq:ot_lag}, we have that $T=\path^\bdot_1$.

This link allows us to associate residual networks with a local action for each layer, which induces a global transportation cost $\Clag$, and taking $p=2$ and the Euclidean norm allows us to refer to the network's kinetic energy. 

\section{\label{sec:gen_sett}General Setting}

In order to better understand the inner workings of a DNN, it is essential to adopt a viewpoint in which the different driving mechanisms become apparent and are decoupled.

\paragraph{Decomposing a DNN}
  We consider the following model of a deep neural network $f$ where computations are separated into the three steps, \textit{i.e.} $f=F \circ T \circ \varphi$ (this is similar to \cite{maxprinciple} and corresponds to the general structure of recent deep models or to the structure of components of a deep model~\cite{he2016a,resnext,wide}):
\begin{enumerate}
    \item \textbf{Dimensionality change:} Starting from an input distribution $\mathcal{D}$ in $\R^n$, a transformation $\varphi$ is applied, transforming it into $\alpha = \varphi_\sharp\mathcal{D}$, a distribution in $\R^d$. This corresponds to the first few layers present in most recent architectures and represents a change of dimensionality. $\varphi$ is known as the \textit{encoder}.
    \item \textbf{Data Transport:} Then $\alpha$ is transformed by a mapping $T: \R^d \to \R^d$, which we see as a transport map. Here, the dimensionality doesn't change and, if this part of the network is a sequence of residual blocks, $T$ can be written as the discretized flow of an ODE. 
    \item \textbf{Task-specific final layers:} A final function $F: \R^d \to \mathcal{Y}$ is applied to $T_\sharp\alpha$ in order to compute the loss $\mathcal{L}$ associated with the task at hand, \textit{e.g.} $F$ could be a perceptron classifier. Like $\varphi$, $F$ is typically made up of a few layers.
\end{enumerate}
The focus of this work is on analyzing the second phase, Data Transport, and we assume that the encoder $\varphi$ is pretrained and fixed (this will be relaxed in some experiments later). To solve a complex non-linear task for which a DNN is needed, the data has to be transformed in a non-trivial way, meaning that this is an essential phase, \textit{e.g.} in the case of classification, $T_\sharp\alpha$ needs to be linearly separable if $F$ is linear. This model is quite general, as many ResNet-based architectures \cite{resnext,wide} alternate modules that change the dimensionality (step 1) and transport modules that keep the dimensionality fixed (step 2) and according to \cite{iterative}, the transport modules have similar behaviour. The model can then be considered as a simplified ResNet, sometimes called a \textit{single representation} ResNet. Note that \cite{model-resolution} finds that networks that keep the same resolution remain competitive.

\paragraph{The set of admissible targets}
As recent neural architectures have systematically achieved near-zero training error~\cite{tradeoff,kernel_learning,ntk,rethinking_generalization}, we place ourselves in this regime, which makes it possible to model this as a hard constraint. For some tasks, this constraint over $T$ is obvious: in a generative setting for example, $T_\sharp\am$ must be equal to some prescribed distribution $\bm$ which is the target of the generation process. But in general, $T$ is less strictly constrained and the condition depends on $F$ and $\mathcal{L}$. This leads us to define a \textit{set of admissible targets} for the task:
\begin{equation}
    \label{adm_maps}
    S_{F,\mathcal{L}} = \{\beta\in\mathcal{P}(\R^d)\ |\ \mathcal{L}(F,\beta) = 0\}
\end{equation}
with $\bm=T_\sharp\am$. In general, $\mathcal{L}$ is fixed while $F$ is learned jointly with $T$. This set is supposed to be non-empty for some $F$ and, in general, it will contain many distributions. The goal of the learning task can then be reformulated as:
\begin{equation}
\label{gen_pbm_class}
    \text{Find } (T,F) \text{ such that } T_\sharp\alpha \in S_{F,\mathcal{L}}
\end{equation}
An important observation is that, even when $S_{F,\mathcal{L}}$ is reduced to a singleton, the problem is still strongly under-constrained and it is possible to obtain many such $(T,F)$ that lead to poor generalization. One can then ask why this is not the case in practice, as good generalization performance is usually achieved.

\paragraph{The case of classification}
Even though our framework is general, we focus our experiments on classification tasks, with $\Lo$ being the cross entropy loss. The task consists in separating $N$ classes. Let us denote $\alpha_i$ the class distributions which are supposed to be distributions in $\mathbb{R}^d$ of mutually disjoint supports, meaning that there is no ambiguity in the class of data points, and such that $\alpha = \sum_i\alpha_i / N$. One wants to find a transformation $T$ of these distributions such that all transported distributions can be correctly classified by a classifier $F$. When $F$ is linear, $S_{F,\mathcal{L}}$ is the set of distributions which have $N$ components that are linearly separated by $F$. Note that we place ourselves in a noiseless ideal setting where perfect classification is possible. 
The question we examine in this work is then twofold: 
\begin{itemize}
    \item What are the properties characterizing mappings reached by standard residual architectures with common hyper-parameters?
    \item Can we find a criteria to $\textit{automatically select}$ mappings with desirable properties in order to improve performance and robustness?
\end{itemize}

\section{\label{sec:emp_ot_an}Empirical Analysis of Transport Dynamics in ResNets}

Before introducing our framework, we conduct an exploratory analysis of the impact of the network's inner dynamics on generalization. We present below two experiments. The first one highlights how good generalization performance is closely related to low transport cost for classification tasks on MNIST and CIFAR10. This cost therefore appears as a natural characterisation of the complexity and disorder of a network. The second experiment, performed on a toy 2D dataset, visualizes the transport induced by the blocks of a ResNet.
 
We consider ResNets where, after encoding, a data point $x_0$ is transported by applying $x_{k+1} = x_k + v_k(x_k)$ for $K$ residual blocks and then classified using $F$. We measure the disorder/complexity of a network by its transport cost which is the sum of the displacements induced by its residual blocks: $\mathcal{C}(v) = \sum_k \| v_k(x_k) \|_2^2$. This quantity corresponds to the kinetic energy of the total displacement.

\paragraph{Transportation cost and generalization on MNIST and CIFAR10.} In order to study the correlation between the transport cost of a residual network and its generalization ability on image data, we train convolutional 9-block ResNets with different initializations~(orthogonal and normal with different gains), for 10-class classification tasks MNIST and CIFAR10. In Figure~\ref{fig:corr}, each point represents a trained network and gives the transport cost $\mathcal{C}$ as a function of the test accuracy of the network. This experiment clearly highlights the strong negative correlation between transport cost and good generalization. This illustrates the importance of the implicit initialization bias and motivates initialization schemes which favour a low kinetic energy. We believe a number of factors contribute to this low energy bias: small initialization gains tend to bias $\| v_k(x_k) \|_2^2$ towards small values, and training using gradient descent does not change this much.

\paragraph{Visualizing network dynamics on 2D toy data.} This experiment provides a 2D visualization  of the transport dynamics inside a network. The task is 2-class classification of a non-linearly separable dataset (two concentric circles, from \texttt{sklearn}) that contains 1000 points with a train-test split of 80$\%$-20$\%$, see Figure \ref{fig:circles} top left. The network is a ResNet containing 9 residual blocks, followed by a fixed linear classifier. Each residual block contains two fully connected layers separated by a batch normalization and a ReLU activation.

\begin{figure}[H]
\centering
\makebox[\textwidth][c]{\includegraphics[width=1.21\textwidth]{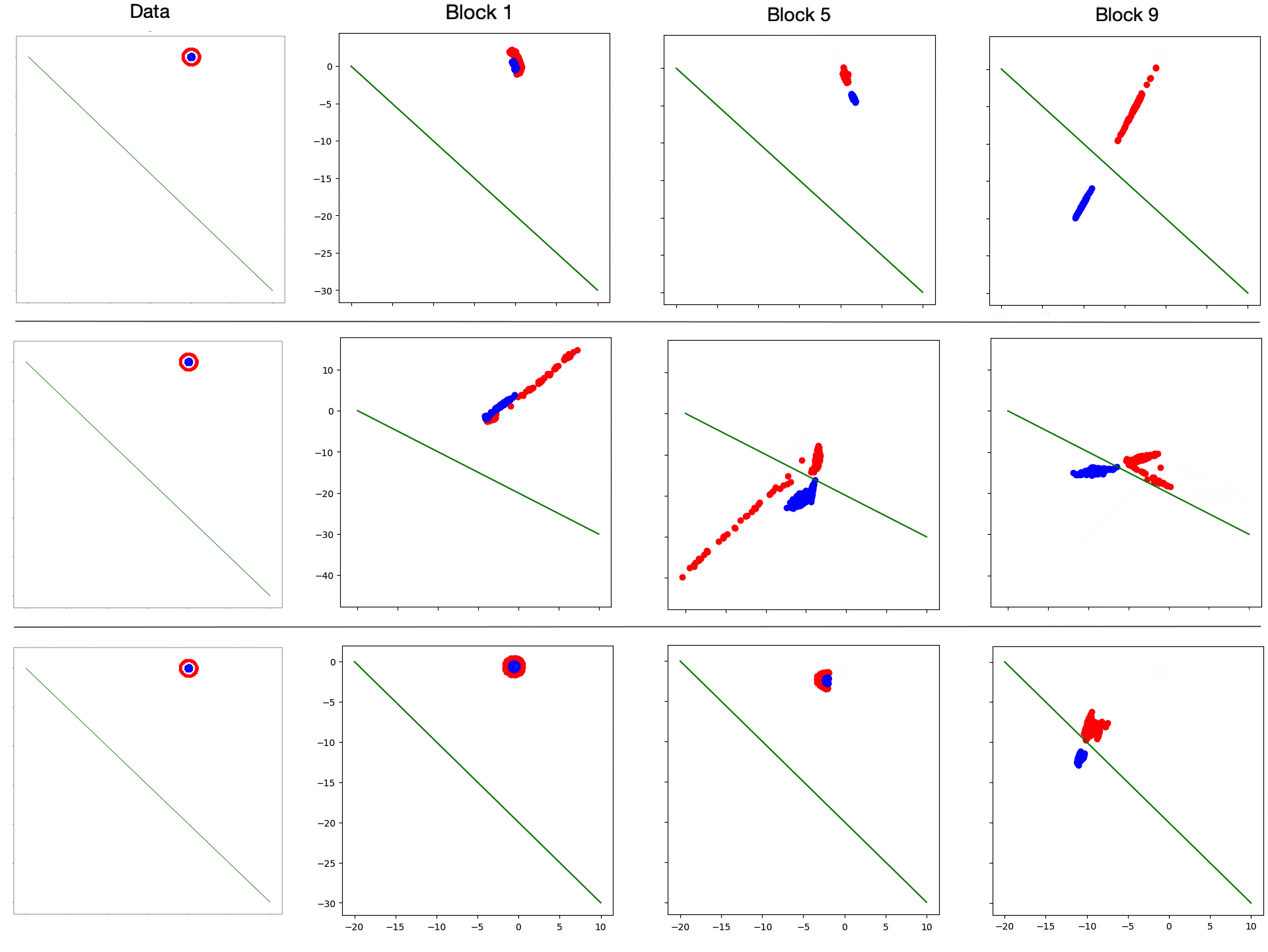}}
\caption{Transformed circles test set by a ResNet9 after blocks 1, 5 and 9 after training; first row with good initialization; second row with a $\mathcal{N}(0,5)$ initialization; third row with a $\mathcal{N}(0,5)$ initialization and the transport cost added to the loss}
\label{fig:circles}
\end{figure}

With the cross-entropy loss alone, the behaviour of a well trained and carefully initialized network achieving $100\%$ test accuracy is illustrated in the first row of Figure \ref{fig:circles}. With a $\mathcal{N}(0,5)$ initialization, significantly bigger than an \say{optimal} initialization, the test accuracy drops to $98\%$~(average of 100 runs) and the transport becomes chaotic~(Figure~\ref{fig:circles}, second row). Adding the transport cost to the loss improves the test accuracy~($99.7\%$ on average) of this badly initialized network and the movement becomes more controlled~(third row of Figure \ref{fig:circles}). Thus, controlling transport improves the behavior and generalization ability of the network. This allows to explicitly control the network whereas implicit biases such as \say{good} initialization rely on heuristics. In Supplementary Material C.4, more experiments show that in other situations that deviate from the ideal setting where the task is perfectly solved, e.g. when using a network which is too large or too small, or a small training set, controlling the transport cost also improves generalization.

\begin{figure}[H]
\centering
\makebox[\textwidth][c]{\includegraphics[width=1.45\textwidth]{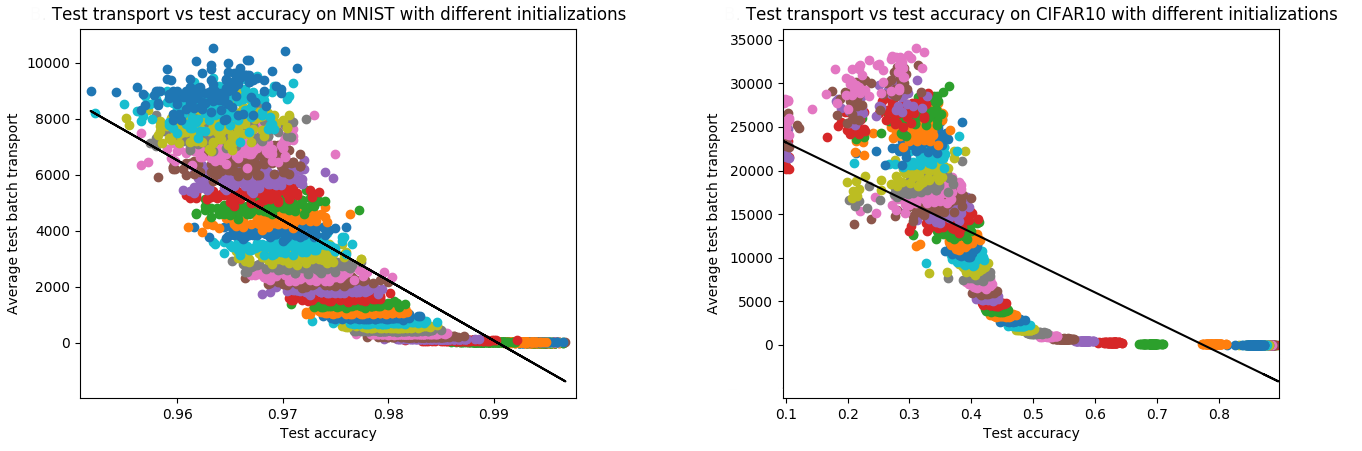}}
\caption{Test transport against test accuracy of ResNet9 models on MNIST (left) and CIFAR10 (right) with fitted linear regressions, where each color indicates a different initialization (either orthogonal or normal with varying gains)}
\label{fig:corr}
\end{figure}

\section{Least Action Principle for Training Neural Networks}
\label{sec:leap}

The previous section has shed some light on the low energy bias of networks as well as on its potential benefits on test accuracy. In this section, we take a step further and make this implicit bias explicit by considering a formulation for training that enforces minimal kinetic energy, closely related to the problem of Optimal Transport. This allows us to prove the existence of minimizers, and exhibit interesting regularity properties of the minimal energy neural networks which may explain good generalization performance. 


\subsection{Formulation}

We consider costs $c(x,y)=\|x-y\|^p$ (where $\|.\|$ is a norm of $\R^d$), with $p>1$, and suppose that $\am\in\mathcal{P}_p(\R^d)$~(the set of absolutely continuous measures on $\R^d$ with finite $p$-th moment). We assume that the space of classifiers is compact, that the loss $\mathcal{L}$ is continuous, that the set $\cup_{F\in \mathcal{F}} S_{F,\mathcal{L}}$ is at a finite $p$-Wasserstein distance $W_p$ from $\am$~(in particular, it is non-empty) and that all its bounded subsets are totally bounded (\textit{i.e.} can be covered by finitely many subsets of any fixed size). These properties depend on the choice of the loss $\mathcal{L}$ and of a class of functions $\mathcal{F}$ for the classifier $F$.

Returning to the transport problem as defined in Section~\ref{subsec:back_ot}, a natural way to select a robust model, given the empirical observations of Section~\ref{sec:emp_ot_an}, is to select, among the maps which transport $\am$ to $S_{F,\mathcal{L}}$ and thus solve the task, one with a minimal transport cost. This gives us the following optimization problem:
\begin{equation}
\label{eq:ot_classif}
\begin{aligned}
& \underset{T, F}{\text{inf}}
& & \mathcal{C}(T) = \int_{\mathbb{R}^d} c(x, T(x)) \dm \am(x)\\
& \text{subject to}
& & T_\sharp\am \in S_{F,\mathcal{L}}
\end{aligned}
\end{equation}

The equivalent dynamical version for $c(x,y)=\|x-y\|^p$ is, as per Section \ref{subsec:back_otres},
\begin{equation}
\label{eq:ot_classif_dyn}
\begin{aligned}
& \underset{v, F}{\text{inf}}
& & \int_0^1 \|v_t\|_{L^p((\path^\bdot_t)_\sharp\am)}^p\, \dm t\\
& \text{subject to}
& & \partial_t\path_t^x = v_t(\path_t^x)\\
& & & \path^\bdot_0= \text{id}\\
& & & (\path^\bdot_1)_\sharp\am \in S_{F,\mathcal{L}}
\end{aligned}
\end{equation}

where $\|v_t\|_{L^p((\path^\bdot_t)_\sharp\am)}^p=\int_{\R^d}\| v_t \|^p \dm (\path^\bdot_t)_\sharp\am$. The result below shows that these two problems are equivalent and that the infima are realized as minima:
\begin{theorem}
\label{th:existence}
The infima of \eqref{eq:ot_classif} and \eqref{eq:ot_classif_dyn} are finite and are realized through a map $T$ which is~(or a velocity field $v$ which induces) an optimal transportation map. When $c(x,y)=\|x-y\|^p$, then \eqref{eq:ot_classif} and \eqref{eq:ot_classif_dyn} are equivalent.
\end{theorem}
\begin{proof}
From the hypothesis above, there exists $\bm\in S_{F,\mathcal{L}}$ at a finite distance from $\am$. Taking any transport map between $\am$ and $\bm$, we see that the infima are finite.

Consider \eqref{eq:ot_classif} and take a minimizing sequence $(T_i,F_i)_i$. Set $\bm_i = (T_i)_\sharp\am$. Then $(\C(T_i))_i$ converges to the infimum which is strictly bounded by $M>0$. Then, by definition, for $i$ large enough, $W^p_p(\am,\bm_i)\leq \C(T_i) \leq M$.
So that $(\bm_i)_i$ is a bounded sequence in $\cup_F S_{F,\mathcal{L}}$. By the hypothesized total boundedness of bounded subsets and as $\mathcal{P}_p(\R^d)$ endowed with $W_p$ is a complete metric space~(see \cite{completeness} for a proof), up to an extraction, $(\bm_i)_i$ converges to $\bm^\star$ in the closure of $\cup_F S_{F,\mathcal{L}}$. Moreover, up to an extraction, $(F_i)_i$ also converges to $F^\star$ by compactness of the class of classifiers. Taking $T^\star$ the OT map between $\am$ and $\bm^\star$ (see Supplementary Material A for existence of OT maps), we then have, by continuity of $\mathcal{L}$,
$$T^\star_\sharp\am = \bm^\star \in S_{F^\star,\mathcal{L}}$$
and $\C(T^\star)\leq \lim \C(T_i)$ by optimality of $T^\star$, which means, since $(\C(T_i))_i$ is a minimizing sequence, that $\C(T^\star)$ minimizes \eqref{eq:ot_classif}. So $(T^\star,F^\star)$ is a minimizer and $T^\star$ is an OT map.

Finally, there exists, by dynamical OT theory ~(Supplementary Material A), a velocity field $v^\star_t$ inducing the OT map between $\am$ and $\bm^\star$ which then gives a minimizer $(v^\star,F^\star)$ for \eqref{eq:ot_classif_dyn}. By the same reasoning, taking a minimizing sequence $(v^{(i)},F_i)_i$ and the induced maps $T_i$ shows that both problems are equivalent.\qed
\end{proof}

Note that uniqueness doesn't hold anymore, as the constraint $ T_\sharp\am \in S_{F,\mathcal{L}}$ in \eqref{eq:ot_classif_dyn} is looser than in standard OT. However, as we show in the following section, the fact that the optimization problems are solved by OT maps will give regularity properties for the models induced by these optimization problems.

\subsection{\label{subsec:ot_reg}Regularity}

Intuitively, the fact that we minimize the energy of the transport map transforming the data is akin to the core idea of Occam's razor: among all the possible networks that correctly solve the the task, the one transforming the data in the simplest way is selected. Moreover, it is possible to show that this optimal transformation is regular: our formulation provides an alternate view on generalization for modern deep learning architectures in the overparametrized regime.

Optimal maps can be as irregular as needed in order to fit the target distribution, however in much the same way as successfully trained DNNs, optimal maps are still surprisingly regular. In a way, they are as regular as possible given the constraints which is exactly the type of flexibility needed. However, the constraints in \eqref{eq:ot_classif} and \eqref{eq:ot_classif_dyn} are looser than in the standard definitions of Optimal Transport. Still, supposing that the input data distribution has a nicely behaved density, namely bounded and of compact support, with the same hypothesis as above, we have the following, which is mainly a corollary of Theorem~\ref{th:existence}:
\begin{proposition}
\label{prop:reg_class}
Consider $T^\star$ the OT map induced by \eqref{eq:ot_classif}~(or \eqref{eq:ot_classif_dyn}) given by Theorem~\ref{th:existence}. Take $X$, respectively $Y$, an open neighborhood of the support of $\am$, respectively of $T^\star_\sharp\am$, then $T^\star$ is differentiable, except on a set of null $\am$ measure.

Additionally, if $T^\star$ doesn't have singularities, there exists $\eta>0$ and $A$, respectively $B$, relatively closed in $X$, respectively $Y$, such that $T^\star$ is $\eta$-Hölder continuous from $X\setminus A$ to $Y\setminus B$. Moreover, if the two densities are smooth, $T^\star$ is a diffeomorphism from $X\setminus A$ to $Y\setminus B$.

\end{proposition}
\begin{proof}
This is a consequence of Theorem~\ref{th:existence}, the hypothesis made in this section and the regularity theorems stated in Supplementary Material B. 
\qed
\end{proof}

There are two main results in Proposition \ref{prop:reg_class}: the first gives $\am$-$a.e.$ differentiability. This is already as strong as might be expected from a classifier:  there are necessarily discontinuities at the frontiers between different classes. The second is even more interesting: it gives Hölder continuity over as large a domain as possible, and even a diffeomorphism if the data distribution is well-behaved enough. We recall that a function $f$ is $\eta$-Hölder continuous for $\eta\in ]0,1]$ if $\exists \ M > 0$ such that $ \|f(x)-f(y)\|\leq M\|x-y\|^\eta$ for all $x,y$. $\eta$ measures the smoothness of $f$, the higher its value, the better. In particular, in the case of classification, this means that the Hausdorff dimension along the frontiers between the different classes is scaled by less than a factor of $1/\eta$ in the transported domain. If the densities are smooth, the dimension even becomes provably smaller by this result.

Intuitively, this means that, in these models, the data is transported in a way that preserves and simplifies the patterns in the input distribution. In the following, we propose a practical algorithm implementing these models and use it for standard classification tasks, showing an improvement over standard models.

\subsection{Practical Algorithm\label{subsec:pract-algo}}

We propose an algorithm for training ResNets using the least action principle by minimizing the kinetic energy. Starting from problem \eqref{eq:ot_classif_dyn} with $p=2$ and the Euclidean norm, we first discretize the differential equation via a forward Euler scheme, which yields ${\path_{k+1}^x = \path_k^x + v_k(\path_k^x)}$. The discretized flow $v_k$ is parameterized by a residual block, giving a standard residual architecture. The residual blocks, along with a classifier $F$, are parametrized by $\theta$. Next, the constraint $(\path^\bdot_1)_\sharp\mathcal{\am} \in S_{F,\mathcal{L}}$ is rewritten as $\mathcal{L}(F, (\path^\bdot_1)_\sharp\mathcal{\am}) = 0$, denoted $\Lo(\theta)=0$ below. Finally, as we only have access to a finite set $\mathcal{X}$ of samples $x$ from $\am$, we use a Monte-Carlo approximation of the integral \textit{w.r.t} the distributions $(\path^\bdot_t)_\sharp\mathcal{\am}$, to obtain:

\begin{mini}
    {\theta}{ \mathcal{C}(\theta) = \sum_{x \in \mathcal{X}}\ \sum_{k=0}^{K-1} \|v_k(\path_k^x) \|_2^2}
    {}{}
    \addConstraint{\path_{k+1}^x}{= \path_k^x + v_k(\path_k^x)}
    \addConstraint{\path^x_0}{= x, \ \forall \ x \in \mathcal{X}}
    \addConstraint{\mathcal{L}(\theta)}{=0}
    \label{eq:ot_classif_discr}
\end{mini}

Is is easy to see that the min-max problem $\min_\theta \max_{\lambda>0} \ \mathcal{C}(\theta) + \lambda \ \Lo(\theta)$ yields the same solution, as the first two constraints are satisfied trivially. If the constraint $\Lo(\theta) = 0$ corresponding to solving the task, which includes the classifier $F$, is not verified, this will cause the second term to grow unbounded, and the solution will thus be avoided by the minimization. This min-max problem can be solved using an iterative approach, starting from some initial $\lambda_0$ and $\theta_0$:

\begin{equation}
\left\{
\begin{aligned}
\theta_{i + 1} &= \arg \min_{\theta} \ \mathcal{C}(\theta) + \lambda_i \ \Lo(\theta) \\
\lambda_{i+1} &= \lambda_i + \tau \ \Lo(\theta_{i+1}) 
\end{aligned}
\right.
\label{eq:uzawa}
\end{equation}

The minimization is done via SGD for a number of steps $s$, where a step means a batch, starting from the previous parameter value $\theta_i$. This algorithm is similar to Uzawa's algorithm used in convex optimization \cite{Santambrogio}. In practice, it is more stable to divide the minimization objective in \eqref{eq:uzawa} by $\lambda_i$, yielding: \\

\renewcommand{\thealgocf}{}
\begin{algorithm}[H]
\DontPrintSemicolon
\caption{{Training neural networks with Least Action Principle (LAP-Net)}}
{\bfseries Input:} Training samples, step size $\tau$, number of steps $s$, initial weight $\lambda_0$ \;
{\bfseries Initialization:} Initialize the parameters $\theta_0$ and set $i=0$ \;
\While{not converged}{
1. Starting from $\theta_i$, perform $s$ steps of stochastic gradient descent: \;
 \ \ \ \ 1.1. $\theta^0_{i+1} = \theta_{i}$ \;
 \ \ \ \ 1.2. $\theta^{l}_{i+1} = \theta^{l-1}_{i+1} - \epsilon ( \nabla \mathcal{C}(\theta^{l-1}_{i+1}) / \lambda_i + \nabla \Lo(\theta^{l-1}_{i+1}))$ for $l$ from $1$ to $s$ \;
 \ \ \ \ 1.3. $\theta_{i+1} = \theta^s_{i+1}$ \;
2. Update the weight $\lambda_{i+1} = \lambda_i + \tau \ \Lo(\theta_{i+1})$ and increment $i \leftarrow i + 1$ \;
}
{\bfseries Output:} Learned parameters $\theta$ \;
\label{alg}
\end{algorithm}

While the high non-convexity makes it difficult to ensure exact optimality, we can still have some induced regularity when reaching a \say{good} local minimum:
\begin{proposition}
Suppose $(F^{\theta^\star},T^{\theta^\star})$ is reached by the optimization algorithm such that $T^{\theta^\star}$ is an $\epsilon-$OT map between $\am$ and its push-forward\footnote{By this, we mean that $\|T^{\theta^\star}-T^\star\|_\infty\leq\epsilon$ where $T^\star$ is the OT map.}. Then we have, with the same notations as in Proposition \ref{prop:reg_class},
\[
\forall x,y\in X\setminus A, \|T^{\theta^\star}(x)-T^{\theta^\star}(y)\|\leq O(\epsilon + \|x-y\|^\eta)
\]
\end{proposition}
\begin{proof}
We simply write the decomposition: $$T^{\theta^\star}(x)-T^{\theta^\star}(y) = T^{\theta^\star}(x)-T^\star(x)+T^\star(x)-T^\star(y)+T^\star(y)-T^{\theta^\star}(y)$$
and use the triangular inequality: the first and third terms are smaller than $\epsilon$ by hypothesis while Hölder continuity applies for the second by Proposition \ref{prop:reg_class}.
\qed
\end{proof}
This shows that minimizing the transport cost still endows the model with some regularity, even in situations where the global minimum is not reached.

\section{\label{sec:expes} Experiments}

\paragraph{MNIST Experiments}

The base model is a ResNet with 9 residual blocks. Two convolutional layers first encode the image of shape $1\times28\times28$ into shape $32\times14\times14$. A residual block contains two convolutional layers, each preceded by a ReLU activation and batch normalization. The classifier is made up of two fully connected layers separated by batch normalization and a ReLU activation. We use an orthogonal initialization \cite{orth} with gain $0.01$. This and all vanilla models and their training regimes are implemented by following closely the cited papers that first introduced them and our method is added over these training regimes. More implementation details are in Supplementary Material C.3.

When using the entire training set, the task is essentially solved ($99.4\%$ test accuracy). We penalize the transport cost as presented in Section \ref{subsec:pract-algo}, using $\lambda_0 = 5$, $\tau = 1$ and $s = 5$. The performance barely drops ($99.3\%$ test accuracy), and we can visualise the preservation of information from the point of view of a pretrained autoencoder (see Supplementary Material C.1). From the experiments in two dimensions, we suspect that adding the transport cost helps when the training set is small. For performance comparisons, we average the highest test accuracy achieved over 30 training epochs (over random orthogonal weight initializations and random subsets of the complete training set). We find that adding the transport cost improves generalization when the training set is very small (Table \ref{tab:mnist}). We see that the improvement becomes more important as the training set becomes smaller and reaches an increase of almost 14 percentage points in the average test accuracy.

\begin{table}[H]
\centering
\begin{tabular}{ccc}
\toprule
Training set size & ResNet & LAP-ResNet~(Ours)  \\ \midrule \\ [-1.1em]
500 & 90.8, [90.4, 91.2] & \textbf{90.9}, [90.7, 91.1]  \\ \\ [-1.1em]
400 & 88.4, [88.0, 88.8] & \textbf{88.4}, [88.0, 88.8]  \\ \\ [-1.1em]
300 & 83.5, [83.0, 84.1] & \textbf{86.2}, [85.8, 86.6]  \\ \\ [-1.1em]
200 & 74.9, [73.9, 75.9] & \textbf{82.0}, [81.5, 82.5]  \\ \\ [-1.1em]
100 & 56.4, [54.9, 58.0] & \textbf{70.0}, [69.0, 71.0]  \\ \bottomrule
\end{tabular}
\caption{Average highest test accuracy and 95$\%$ confidence interval of ResNet9 over 50 instances on MNIST with training sets of different sizes (in $\%$)} 
\label{tab:mnist}
\end{table}

\paragraph{CIFAR10 Experiments}

We run the same experiments on CIFAR10. The architecture is exactly the same except that the encoder transforms the input which is of shape $3\times32\times32$ into shape $100\times16\times16$. For our method, we use $\lambda_0 = 0.1$, $\tau = 0.1$ and $s = 50$. We average the highest test accuracy achieved over 200 training epochs over random orthogonal weight initializations and random subsets of the complete train set. Here, we find that adding the transport cost helps for all sizes of the train set~(which has 50 000 images in total). The increase in average precision becomes more important as the train set becomes smaller~(Table \ref{tab:cifar10}).

\begin{table}[H]
\centering
\begin{tabular}{ccc}
\toprule
Training set size & ResNet & LAP-ResNet~(Ours)  \\ \midrule \\ [-1.1em]
50 000 & 91.49, [91.40, 91.59]  & \textbf{91.94}, [91.84, 92.04] \\ \\[-1.1em]
30 000 & 88.61, [88.47, 88.75] & \textbf{89.41}, [89.31, 89.50] \\ \\[-1.1em]
20 000 & 85.73, [85.59, 85.87] & \textbf{86.74}, [86.61, 86.87] \\ \\[-1.1em]
10 000 & 79.25, [79.00, 79.49] & \textbf{80.90}, [80.74, 81.06] \\ \\[-1.1em]
5 000 & 70.32, [70.00, 70.63] & \textbf{72.58}, [72.36, 72.79] \\ \\[-1.1em]
4 000 & 67.80, [67.55, 68.07] & \textbf{70.12}, [69.81, 70.42] \\ \\[-1.1em]
 \bottomrule
\end{tabular}
\caption{Average highest test accuracy and 95$\%$ confidence interval of ResNet9 over 20 instances on CIFAR10 with training sets of different sizes (in $\%$)} 
\label{tab:cifar10}
\end{table}

\paragraph{CIFAR100 experiments}

On CIFAR100, results using a ResNet are in Supplementary Material C.2. We also used the ResNeXt~\cite{resnext} architecture: the residual block of a ResNeXt applies $x + \sum_i w_i(x)$ with the functions $w_i$ having the same architecture but independent weights, followed by a ReLU activation. We used the ResNeXt-50-32$\times$4d architecture detailed in \cite{resnext}. This is a much bigger and state-of-the-art network, as compared with the single representation ResNet used so far. It also extends the experimental results beyond the theoretical framework in three ways: the embedding dimension changes between the residual blocks, a block applies $x_{k+1} = \texttt{ReLU}(x_k + \sum_i w_{k,i}(x_k))$ and the encoder is no longer fixed. We found that penalizing $\sum_i w_{k,i}(x_k)$ or $x_{k+1} - x_k$ is essentially equivalent. Table \ref{tab:cifar100-resnext} shows consistent accuracy gains as our method (with $\lambda_0 = 1$, $\tau = 0.1$ and $s = 5$) corrects a slight overfitting of the bigger ResNeXt compared to ResNet. 

\begin{table}[H]
\centering
\begin{tabular}{ccc}
\toprule
Training set size & ResNeXt & LAP-ResNeXt~(Ours)  \\ \midrule \\ [-1.1em]
50 000 & 72.97, [71.79, 74.14]  & \textbf{76.11}, [75.32, 76.89] \\ \\[-1.1em]
25 000 & 62.55, [60.18, 64.92] & \textbf{64.11}, [62.25, 65.96] \\ \\[-1.1em]
12 500 & 45.90, [43.16, 48.67] & \textbf{48.23}, [46.39, 50.07] \\ \bottomrule
\end{tabular}
\caption{Average highest test accuracy and 95$\%$ confidence interval of ResNeXt50 over 10 instances on CIFAR100 with training sets of different sizes (in $\%$)} \label{tab:cifar100-resnext}
\end{table}

An important observation is that adding the transport cost significantly reduces the variance in the results. This is expected as the model becomes more constrained and can be seen as an advantage, especially in cases where the results vary more with the initialization~(\textit{e.g.} transfer learning). This is illustrated by the width of the $95\%$ confidence intervals in the tables above often becoming narrower when the transport cost is penalized. Finally, we could also have considered a relaxation of the optimization program by considering a fixed weight $\lambda$, which provides a simpler and quite competitive benchmark (see Supplementary Material C.2). The training's progress is shown there as well, and we see that the training is not slowed down by our method.

\section{Related work}

That ResNets~\cite{he2016b,he2016a} are naturally biased towards minimally transforming their input, especially for later blocks and deeper networks, is already shown in \cite{iterative}, which found that earlier blocks learn new representations while later blocks only slowly refine those representations. \cite{hauser} found that the deeper the network the more its blocks minimally move their input. Both were inspirations for this work. 
The ODE point of view of ResNets has inspired new architectures~\cite{reversible,imexnet,lm,pde}. Others were inspired by numerical schemes to improve stability, e.g. \cite{reversible} add a penalty term that encourages the weights to vary smoothly from layer to layer and \cite{small-step} replicate an Euler scheme and study the effect of diminishing the discretization step-size. More recently, \cite{noderobustness} accelerate the training of \cite{node}'s model for generative tasks using the link with dynamical transport. But most often, regularization is achieved by penalization of the weights~(e.g. spectral norm regularization~\cite{spectral}, smoothly varying weights \cite{reversible}).

OT theory was used in~\cite{dae} to analyse deep gaussian denoising autoencoders~(not necessarily implemented through residual networks) as transport systems. In the continuous limit, they are shown to transport the data distribution so as to decrease its entropy. Closer to this work, the dynamical formulation of OT is used in \cite{oudt} for the problem of unsupervised domain translation. 

\section{Discussion and Conclusion}

In this work, we have studied the behavior of ResNets by adopting a dynamical systems perspective. This viewpoint leverages the vast literature in this field. 

More specifically, we have analyzed ResNets' complexity through the lens of the transport cost induced by the data displacement across the model's blocks. We find that due to a certain number of factors, this transport cost is biased towards small values. Moreover, this cost is negatively correlated to test accuracy, which has brought us to consider explicitly minimizing it. This leads us to present a novel generic formulation for training neural networks, based on the least action principle, closely related to the problem of Optimal Transport: amongst all the neural networks that correctly solve the task, select the one that transforms the data with the lowest cost. Note that even though we have only considered residual networks as they induce an ODE flow, this framework can be applied to any architecture by considering the static formulation~\eqref{eq:ot_classif} of the problem. 

We have proven general results of existence and regularity for models trained within our framework, studied their behaviour in low-dimensional settings when compared to vanilla models and shown their efficiency on standard classification tasks. We also found that the training is stabilized in an adaptive fashion without being slowed down. 

An important property of our method which is yet to be tested and is hinted at by the regularity results and by the lower variance in the performances is the robustness of the models, more specifically in adversarial contexts. This will be one important venue of future work. Another interesting avenue of research would be to experiment with alternative transportation costs.  

\bibliographystyle{abbrv}
\bibliography{main-all}

\begin{thebibliography}{10}

\bibitem{ambrosio}
L.~Ambrosio, N.~Gigli, and G.~Savare.
\newblock {\em Gradient Flows in Metric Spaces and in the Space of Probability
  Measures}.
\newblock Birkhäuser Basel, 2005.

\bibitem{tradeoff}
M.~Belkin, D.~Hsu, S.~Ma, and S.~Mandal.
\newblock Reconciling modern machine-learning practice and the classical
  bias{\textendash}variance trade-off.
\newblock {\em PNAS}, 2019.

\bibitem{kernel_learning}
M.~Belkin, S.~Ma, and S.~Mandal.
\newblock To understand deep learning we need to understand kernel learning.
\newblock In {\em ICML}, 2018.

\bibitem{brenier}
J.~Benamou and Y.~Brenier.
\newblock A computational fluid mechanics solution to the monge-kantorovich
  mass transfer problem.
\newblock {\em Numerische Mathematik}, 2000.

\bibitem{batchnorm}
N.~Bjorck et~al.
\newblock Understanding batch normalization.
\newblock In {\em NIPS}, 2018.

\bibitem{completeness}
F.~Bolley.
\newblock Separability and completeness for the wasserstein distance.
\newblock In {\em Séminaire de Probabilités XLI}. Springer, 2008.

\bibitem{reversible}
B.~Chang et~al.
\newblock Reversible architectures for arbitrarily deep residual neural
  networks.
\newblock In {\em AAAI}, 2018.

\bibitem{node}
R.~Chen, Y.~Rubanova, J.~Bettencourt, and D.~K. Duvenaud.
\newblock Neural ordinary differential equations.
\newblock In {\em NIPS}, 2018.

\bibitem{oudt}
E.~de~Bézennac, I.~Ayed, and P.~Gallinari.
\newblock Optimal unsupervised domain translation.
\newblock {\em arXiv}, 2019.

\bibitem{simple_functions}
G.~De~Palma, B.~Kiani, and S.~Lloyd.
\newblock Random deep neural networks are biased towards simple functions.
\newblock In {\em NIPS}, 2019.

\bibitem{feynman}
R.~P. Feynman.
\newblock The principle of least action in quantum mechanics.
\newblock In {\em Feynman's Thesis - A New Approach to Quantum Theory}. World
  Scientific Publishing, 2005.

\bibitem{figalli2017monge}
A.~Figalli.
\newblock {\em The Monge-Amp{\`e}re Equation and Its Applications}.
\newblock Zurich lectures in advanced mathematics. European Mathematical
  Society, 2017.

\bibitem{thermo}
V.~Garcia-Morales, J.~Pellicer, and J.~Manzanares.
\newblock Thermodynamics based on the principle of least abbreviated action.
\newblock {\em Annals of Physics}, 2008.

\bibitem{gray}
C.~G. Gray.
\newblock {P}rinciple of least action.
\newblock {\em Scholarpedia}, 2009.

\bibitem{imexnet}
E.~Haber, K.~Lensink, E.~Treister, and L.~Ruthotto.
\newblock {IMEX}net a forward stable deep neural network.
\newblock In {\em ICML}, 2019.

\bibitem{statistical_learning}
T.~Hastie, R.~Tibshirani, and J.~Friedman.
\newblock {\em The Elements of Statistical Learning}.
\newblock Springer-Verlag, 2001.

\bibitem{hauser}
M.~Hauser.
\newblock On residual networks learning a perturbation from identity.
\newblock {\em arXiv}, 2019.

\bibitem{he2016b}
K.~He, X.~Zhan, S.~Ren, and J.~Sun.
\newblock Identity mappings in deep residual networks.
\newblock In {\em ECCV}, 2016.

\bibitem{he2016a}
K.~He, X.~Zhang, S.~Ren, and J.~Sun.
\newblock Deep residual learning for image recognition.
\newblock In {\em CVPR}, 2016.

\bibitem{ntk}
A.~Jacot, F.~Gabriel, and C.~Hongler.
\newblock Neural tangent kernel: Convergence and generalization in neural
  networks.
\newblock In {\em NIPS}, 2018.

\bibitem{iterative}
S.~Jastrzebski et~al.
\newblock Residual connections encourage iterative inference.
\newblock In {\em ICLR}, 2018.

\bibitem{maxprinciple}
Q.~Li, L.~Chen, C.~Tai, and W.~E.
\newblock Maximum principle based algorithms for deep learning.
\newblock {\em JMLR}, 2018.

\bibitem{lm}
Y.~Lu, A.~Zhong, Q.~Li, and B.~Dong.
\newblock Beyond finite layer neural networks: Bridging deep architectures and
  numerical differential equations.
\newblock In {\em ICML}, 2018.

\bibitem{mtw}
X.~Ma, N.~Trudinger, and X.~Wang.
\newblock Regularity of potential functions of the optimal transportation
  problem.
\newblock {\em Archive for Rational Mechanics and Analysis}, 2005.

\bibitem{double}
P.~Nakkiran et~al.
\newblock Deep double descent: Where bigger models and more data hurt.
\newblock In {\em ICLR}, 2020.

\bibitem{sensitivity}
R.~Novak, Y.~Bahri, D.~A. Abolafia, J.~Pennington, and J.~Sohl{-}Dickstein.
\newblock Sensitivity and generalization in neural networks: an empirical
  study.
\newblock In {\em ICLR}, 2018.

\bibitem{cot}
G.~Peyre and M.~Cuturi.
\newblock {\em Computational {Optimal} {Transport}}.
\newblock Now Publishers, 2019.

\bibitem{spectral_bias}
N.~Rahaman et~al.
\newblock On the spectral bias of neural networks.
\newblock In {\em ICML}, 2019.

\bibitem{pde}
L.~Ruthotto and E.~Haber.
\newblock Deep neural networks motivated by partial differential equations.
\newblock {\em J Math Imaging Vis}, 2020.

\bibitem{model-resolution}
M.~Sandler et~al.
\newblock Non-discriminative data or weak model? on the relative importance of
  data and model resolution.
\newblock In {\em ICCVW}, 2019.

\bibitem{Santambrogio}
F.~Santambrogio.
\newblock {\em Optimal transport for Applied Mathematicians}.
\newblock Birkhäuser, 2015.

\bibitem{orth}
A.~M. Saxe, J.~L. Mcclelland, and S.~Ganguli.
\newblock Exact solutions to the nonlinear dynamics of learning in deep linear
  neural network.
\newblock In {\em ICLR}, 2014.

\bibitem{dae}
S.~Sonoda and N.~Murata.
\newblock Transport analysis of infinitely deep neural network.
\newblock {\em JMLR}, 2019.

\bibitem{villani}
C.~Villani.
\newblock {\em Optimal Transport: Old and New}.
\newblock Springer-Verlag, 2008.

\bibitem{weinan}
E.~Weinan.
\newblock A proposal on machine learning via dynamical systems.
\newblock {\em Commun. Math. Stat}, 2017.

\bibitem{resnext}
S.~Xie et~al.
\newblock Aggregated residual transformations for deep neural networks.
\newblock In {\em CVPR}, 2017.

\bibitem{noderobustness}
H.~Yan, J.~Du, V.~Tan, and J.~Feng.
\newblock On robustness of neural ordinary differential equations.
\newblock In {\em ICLR}, 2020.

\bibitem{spectral}
Y.~Yoshida and T.~Miyato.
\newblock Spectral norm regularization for improving the generalizability of
  deep learning.
\newblock {\em arXiv}, 2017.

\bibitem{wide}
S.~Zagoruyko and N.~Komodakis.
\newblock Wide residual networks.
\newblock In {\em BMVC}, 2016.

\bibitem{rethinking_generalization}
C.~Zhang et~al.
\newblock Understanding deep learning requires rethinking generalization.
\newblock In {\em ICLR}, 2017.

\bibitem{small-step}
J.~Zhang et~al.
\newblock Towards robust resnet: A small step but a giant leap.
\newblock In {\em IJCAI}, 2019.

\end{thebibliography}

\clearpage
\appendix

\section{\label{app_sec:ot_theory}Some Elements of Optimal Transport Theory}

We state here the most important results of Optimal Transport theory and its dynamical formulation. Our main reference is \cite{Santambrogio}. \cite{villani} is another classical reference. The dynamical formulation of OT has been of great importance, both theoretically and practically. It stems mainly from the work of Benamou and Brenier~\cite{brenier}.

\subsection{\label{subsec:ot_basics}Optimal Transport}

OT studies the task of \say{transporting} mass from one configuration to another while minimizing the effort as described by a certain ground cost $c$. Let $\alpha$ and $\beta$ be two absolutely continuous distributions. The Monge formulation of OT is:
\begin{mini}
    {T}{ \mathcal{C}(T) = \int_{\mathbb{R}^d} c(x, T(x))\dm \alpha(x)}
    {}{}
    \addConstraint{T_\sharp \alpha }{=\beta}
    \label{eq:monge_app}
\end{mini}

We then have the following result, proven for example in Theorem~1.17 of \cite{Santambrogio}, which gives a condition on the cost under which problem~\eqref{eq:monge_app} has a unique minimum.
\begin{theorem}
    $\alpha$, $\beta$ absolutely continuous measures on $\mathbb{R}^d$. If $c(x,y)=h(x-y)$, with $h$ strictly convex, then there exists a unique $T$ such that $\mathcal{C}(T)$ is minimal.
\end{theorem}

\subsection{\label{subsec:dyn_form}Dynamical Formulation}

Instead of directly pushing samples of $\am$ to $\bm$ in $\mathbb{R}^d$, we can view $\am$ and $\bm$ as points in a space of measures, and consider trajectories from $\am$ to $\bm$ in this space.A way to transport the probability mass from $\am$ to $\bm$ is a curve between two points in this space. The curve corresponding to the optimal mapping is the \textit{shortest} one, in other words it is the \textit{geodesic curve} between $\am$ and $\bm$. More formally, we introduce the \textit{Wasserstein metric space} $\mathbb{W}_p(\mathbb{R}^d)$, \textit{i.e.} the space of absolutely continuous measures of $\mathbb{R}^d$ with finite $p$-th moment endowed with the Wasserstein distance:
\[
W_p(\mu,\nu) = \min_{T_\sharp \mu=\nu} \C(T)^{\frac{1}{p}}
\]
when costs $c(x,y) = \|x-y\|_q^p$ are considered, for $q,p > 1$. The OT map can then be seen as a trajectory of minimal length between $\am$ and $\bm$, in other words a \textit{geodesic}. The following result~(from Theorem 5.27 of \cite{Santambrogio}) motivates this approach:
\begin{theorem}
$\mathbb{W}_p$ is a geodesic space, meaning that, for any measures $\mu, \nu\in\mathbb{W}_p$, there exists a geodesic curve $(\mu_t)_{ t \in [0, 1]}$ between $\mu$ and $\nu$.
\end{theorem}

Thus, according to this result, finding the optimal mapping between two distributions amounts to finding a curve of minimal length in a certain abstract measure space. However, it still does not provide much in the way of a practically useful algorithm. The following theorem makes a formal link with fluid dynamics and basically states that moving probability masses from one distribution to another is the same as moving fluid densities from one configuration to another under a certain velocity field~\cite{Santambrogio}:
\begin{theorem} \label{thm:dynamic}
Given $\am$ and $\bm$ absolutely continuous \textit{w.r.t.} the Lebesgue measure and $(\mu_t)_{t \in [0, 1]}$ the geodesic curve with $\mu_0 = \am$ and $\mu_1 = \bm$, we can associate a vector field $v_t\in L^p(\mu_t)$ that solves the \textit{continuity equation}\footnote{$\partial_t$ is the partial derivative operator \textit{w.r.t.} variable $t$, and $\nabla \cdot$ the divergence operator \textit{w.r.t.} space.}:
\[
\partial_t\mu_t + \nabla\cdot(\mu_t v_t) = 0
\]
with:
\[
W^p_p(\am,\bm) = \int_0^1\|v_t\|^p_{L^p(\mu_t)}\mathrm{dt}
\]
\end{theorem}

In other words, the geodesic curve $(\mu_t)_{t \in [0, 1]}$ between the two distributions and the minimal energy velocity vector field $v$ solve the continuity equation. Moreover, the energy along this path is precisely equal to the Wasserstein distance $W^p_p(\am, \bm)$. If this vector field of minimal energy $v$ could be obtained, probability mass could be displaced according to the flow defined by the continuity equation, and the geodesic curve could be retrieved. Thus, we can reformulate the problem as a problem of optimal control, where $v$ is the control variate:
\begin{mini}
    {v}{ \Cdyn(v) = \int_0^1 \|v_t\|^p_{L^p(\mu_t)}\dm t}
    {}{}
    \addConstraint{\partial_t \mu_t + \nabla \cdot (\mu_t v_t)}{= 0, \mu_{0}= \am, \mu_{1}= \bm}
    \label{eq:dyn_supp}
\end{mini}

\section{\label{app_sec:regularity}Regularity of Optimal Transport Maps}

In this section, we recall some classical and more recent results of regularity for Optimal Transport mappings. This is an intricate subject and the problem was open for some time after OT theory had been established. The most important results have been established through the study of the Monge-Ampère equation by Caffarelli then De Philippis and Figalli. Extensions for larger families of costs were developed by Ma, Trudinger and Wang~\cite{mtw} but this is out of the scope of this work. In particular, Theorem~6.27 of \cite{ambrosio} gives a classical almost-everywhere regularity result:
\begin{theorem}
If $c(x,y) = \|x-y\|^p$ for $p>1$, and $\am$ and $\bm$ have compact supports with $d(\text{supp}(\am),\text{supp}(\bm))>0$, then the optimal transportation map $T$ between $\am$ and $\bm$ is $\am-a.e.$ differentiable and its Jacobian $\nabla T(x)$ has non-negative eigenvalues $\am-a.s.$
\end{theorem}

More recently, results summarized below, which correspond to Theorems~4.23, 4.24 and Remark~4.25 of \cite{figalli2017monge}, state that the optimal transportation map has one degree of regularity more than the initial transported density:

\begin{theorem}
Suppose there are $X,Y$, bounded open sets, such that the densities of $\am$ and $\bm$ are null in their respective complements and bounded away from zero and infinity over them respectively. Then, if $Y$ is convex, there exists $\eta>0$ such that the OT map $T$ between $\am$ and $\bm$ is $C^{0,\eta}$ over $X$. If $Y$ isn't convex, there exists two relatively closed sets $A,B$ in $X,Y$ respectively, such that $T\in C^{0,\eta}(X\setminus A,Y\setminus B)$, where $A$ and $B$ are of null Lebesgue measure.

Moreover, if the densities are in $C^{k,\eta}$, then $C^{0,\eta}$ can be replaced by $C^{k+1,\eta}$ in the conclusions above. In particular, if the densities are smooth, then the transport map is a diffeomorphism~(between the reduced input and target domains if the target support is not convex). 
\end{theorem}

\section{\label{app_sec:results}Additional Results}

\subsection{\label{app_subsec:decodings}Visualization of the Transport on MNIST}

If we pretrain an autoencoder on MNIST, we can use its encoder as the encoder of the ResNet and freeze it during training. This makes it possible to visualize the transport of the data by decoding, using the pretrained decoder, the output of each residual block. We show this below on MNIST. In Figure \ref{fig:decodings-mnist-notra}, we see the decodings of a basic ResNet trained to achieve $99.4\%$ test accuracy. 

\begin{figure}[H]
\centering
\makebox[\textwidth][c]{\includegraphics[width=1.25\textwidth]{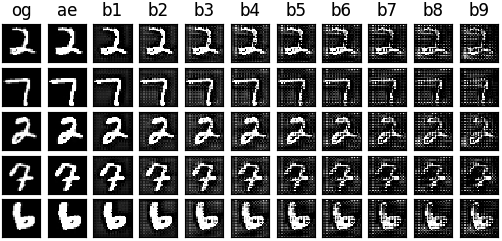}}
\caption{Decodings of the internal representations (the outputs of the residual blocks) after training a ResNet9 on MNIST (og: original image, ae: encoding, b1: output of block 1...)}
\label{fig:decodings-mnist-notra}
\end{figure}

We add the transport cost with $\lambda_0 = 5$, $\tau = 1$ and $s = 5$. The performance barely drops ($99.3\%$ test accuracy) but we can see in Figure \ref{fig:decodings-mnist-uzawa} the effect of the regularization as the decodings change much less. 

\begin{figure}[H]
\centering
\makebox[\textwidth][c]{\includegraphics[width=1.25\textwidth]{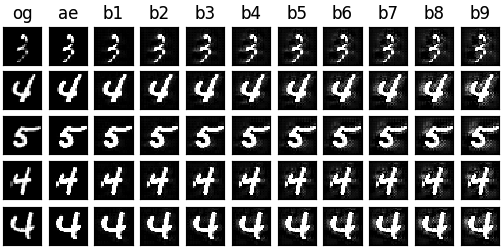}}
\caption{Decodings of the internal representations (the outputs of the residual blocks) after training a LAP-ResNet9 on MNIST (og: original image, ae: encoding, b1: output of block 1...)}
\label{fig:decodings-mnist-uzawa}
\end{figure}

\subsection{\label{app_subsec:lambdaresults}Additional Results with Fixed $\lambda$}

In this section, we show some additional experimental results with a model where, instead of using an adaptive optimization algorithm, we simply take the transport cost as a regularizer, thus giving us a minimization objective:
\[
\mathcal{L}(\theta) +  \lambda \ \ \mathcal{C}(\theta)
\]
This is an easier and more straightforward method, simply considering a relaxed constraint in the optimization problem. Aside from the advantage of simpler implementation, it allows for easier fine-tuning of the regularization hyper-parameter which is useful when the datasets and networks are big. The adaptivity is lost but this still leads to better test accuracy than the non-regularized networks. Results on the same tasks as in Section 6 are below.

\begin{table}[H]
\centering
\hspace*{-2.2cm} 
\begin{tabular}{cccc}
\toprule
Training set size & ResNet & LAP-ResNet & Regularized ResNet, $\lambda = 0.2$  \\ \midrule \\ [-1.1em]
50 000 & 91.49, [91.40, 91.59] & \textbf{91.94}, [91.84, 92.04] & 91.36, [91.28, 91.44] \\ \\[-1.1em]
30 000 & 88.61, [88.47, 88.75] & \textbf{89.41}, [89.31, 89.50] & 88.50, [88.38, 88.61] \\ \\[-1.1em]
20 000 & 85.73, [85.59, 85.87] & \textbf{86.74}, [86.61, 86.87] & 85.82, [85.70, 85.93] \\ \\[-1.1em]
10 000 & 79.25, [79.00, 79.49] & \textbf{80.90}, [80.74, 81.06] & 80.15, [80.02, 80.28] \\ \\[-1.1em]
 5 000 & 70.32, [70.00, 70.63] & \textbf{72.58}, [72.36, 72.79] & 72.03, [71.71, 72.34] \\ \\[-1.1em]
 4 000 & 67.80, [67.55, 68.07] & \textbf{70.12}, [69.81, 70.42] & 69.64, [69.35, 69.94] \\ \\[-1.1em]
 1 000 & 49.22, [48.69, 49.74] & \textbf{51.14}, [50.69, 51.59] & 50.38, [49.92, 50.82] \\ \\[-1.1em]
   500 & 41.55, [41.14, 41.96] & \textbf{42.92}, [42.54, 43.29] & 42.30, [41.88, 42.73] \\ \\[-1.1em]
   100 & 26.98, [25.98, 27.97] & 25.34, [24.63, 26.10] & \textbf{27.53}, [26.59, 28.47] \\ \bottomrule
\end{tabular}
\caption{Average highest test accuracy and 95$\%$ confidence interval of ResNet9 over 20 instances on CIFAR10 with training sets of different sizes (in $\%$)} 
\label{tab:cifar10-comp}
\end{table}

\begin{table}[H]
\centering
\hspace*{-2.8cm} 
\begin{tabular}{cccc}
\toprule
Training set size & ResNet & LAP-ResNet & Regularized ResNet, $\lambda \in \{0.05, 0.2\} $  \\ \midrule \\ [-1.1em]
50 000 & 72.32, [72.08, 72.56] & 72.43, [72.25, 72.61] & \textbf{72.62}, [72.41, 72.83] \\ \\[-1.1em]
25 000 & 64.34, [64.10, 64.57] & 64.34, [64.11, 64.58] & \textbf{64.76}, [64.52, 65.00] \\ \\[-1.1em]
10 000 & 49.27, [48.84, 49.69] & \textbf{50.57}, [50.34, 50.80] & 50.46, [50.19, 50.72] \\ \\[-1.1em]
5 000 & 34.74, [33.90, 35.58] & 37.97, [37.68, 38.27] & \textbf{38.44}, [37.99, 38.89] \\ \\[-1.1em]
1 000 & 15.66, [15.23, 16.08] & \textbf{16.42}, [16.10, 16.75] & 16.03, [15.55, 16.52] \\ \bottomrule
\end{tabular}
\caption{Average highest test accuracy and 95$\%$ confidence interval of ResNet9 over 10 instances on CIFAR100 with training sets of different sizes (in $\%$)} \label{tab:cifar100-comp}
\end{table}

\begin{table}[H]
\centering
\hspace*{-2.4cm} 
\begin{tabular}{cccc}
\toprule
Training set size & ResNeXt & LAP-ResNeXt & Regularized ResNeXt, $\lambda = 0.01$  \\ \midrule \\ [-1.1em]
50 000 & 72.97, [71.79, 74.14] & \textbf{76.11}, [75.32, 76.89] & 75.96, [74.92, 77.01]  \\ \\[-1.1em]
25 000 & 62.55, [60.18,64.92] & \textbf{64.11}, [62.25, 65.96] & 64.10, [62.36, 65.84] \\ \\[-1.1em]
12 500 & 45.90, [43.16, 48.67] & \textbf{48.23}, [46.39, 50.07] & 47.77, [45.93, 49.62] \\ \bottomrule
\end{tabular}
\caption{Average highest test accuracy and 95$\%$ confidence interval of ResNeXt50 over 10 instances on CIFAR100 with training sets of different sizes (in $\%$)} \label{tab:resnext-cifar100}
\end{table}

Finally, we point out that the least action principle acts by speeding up training in the first epochs as seen for the training of ResNeXt50 models on CIFAR100 in Figure \ref{fig:training}. Batch training times are similar for the 3 models in Figure  \ref{fig:training} on the same hardware (around 0.7 seconds).

\begin{figure}[H]
\centering
\makebox[\textwidth][c]{\includegraphics[width=1.02\textwidth]{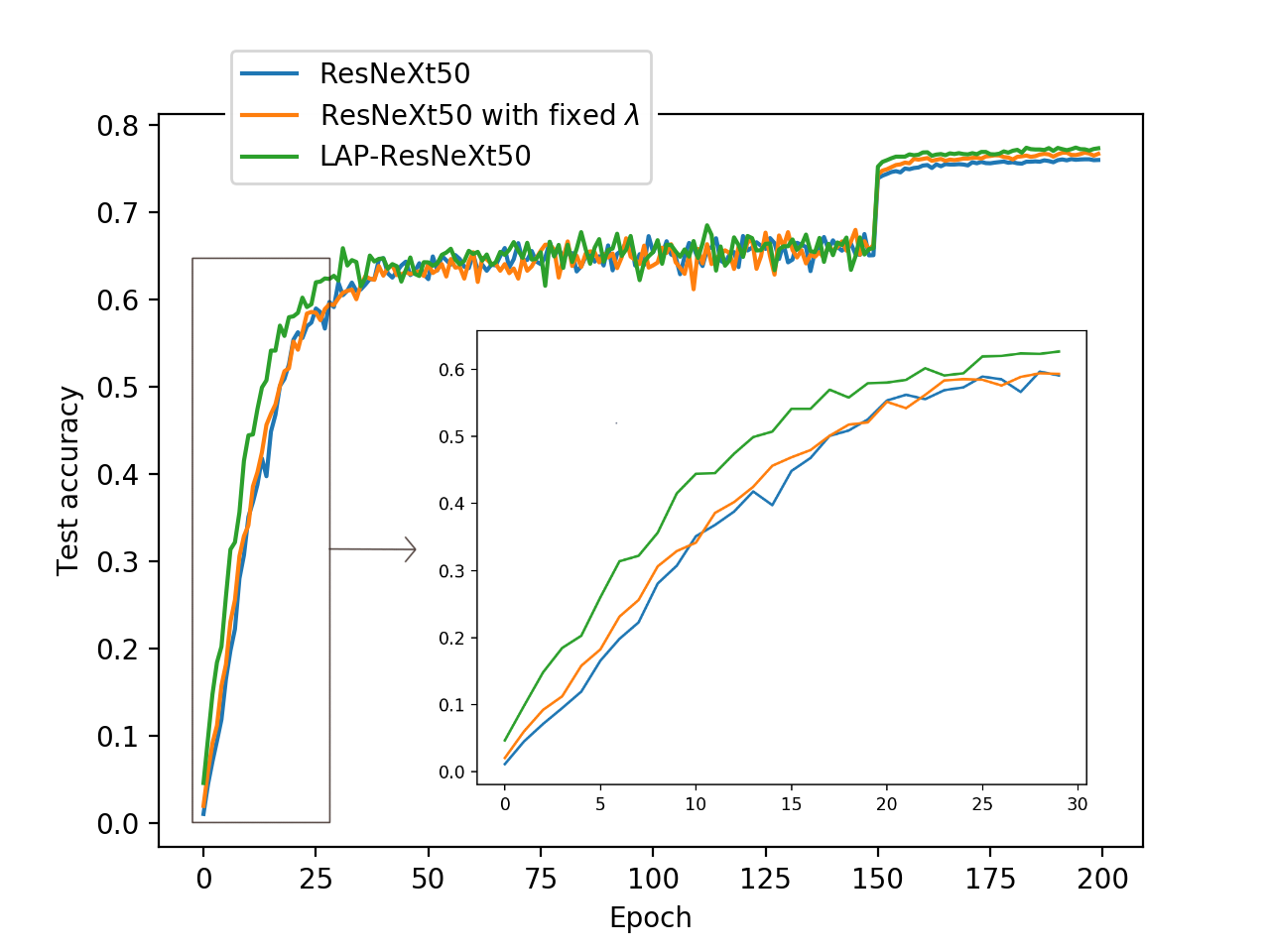}}
\caption{Test accuracy during training of ResNeXt50 models on CIFAR100}
\label{fig:training}
\end{figure}

\subsection{\label{app_subsec:implem}Implementation Details}

These are further implementation details about the experiments in Sections 6 and \ref{app_subsec:lambdaresults}. Orthogonal initialization with gain $0.01$ is used for all ResNet models. Kaiming initialization is used for all ResNeXt models. SGD is used for training all models. The momentum is always set to $0.9$ and weight decay to $0.0001$. For ResNet models, the learning rate is $0.01$ and is divided by 5 at epochs 120, 160 and 200 (when the training goes that far). For ResNeXt models, the learning rate is $0.1$ and is divided by 10 at epochs 150, 225 and 250. Batch size is 128 for all experiments. Architectures of ResNet \cite{he2016b} and ResNeXt \cite{resnext} blocks are standard and exactly as in the references. The ResNets used are single representation ResNets (i.e. containing one residual stage) with 9 blocks. The ResNeXt architecture used is the ResNeXt-50-32$\times$4d from \cite{resnext}.

\subsection{\label{app_subsec:2dresults}Additional Results on 2D Toy Data}

Here is a comparison of our method with batch normalization (BN), which is known to impact the loss surface's geometry ~\cite{batchnorm}.We find that our method cooperates well with BN to improve test accuracy on the same 2D task as in Section 4 when the model is too small (1 block, Table \ref{tab:1000p1b}), too big (100 blocks, Table \ref{tab:1000p100b}), badly initialized ($\mathcal{N}(0,5)$ initialization, Table \ref{tab:biginit}) and when the dataset is small (50 points, Table \ref{tab:50p9b}). LAP-ResNets use $\lambda_0 = 0.1$, $\tau = 0.1$ and $s = 5$.

\begin{table}[H]
\centering
\hspace*{-0.8cm} 
\begin{tabular}{ccc}
\toprule
& No batch normalization & Batch normalization  \\ \midrule \\ [-1.1em]
ResNet & 76.6, [73.1, 80.2] &  75.4, [72.3, 78.6] \\ \\[-1.1em]
Regularized ResNet, $\lambda = 0.005$ & 76.5, [73.0, 80.0] &  75.6, [72.2, 78.9] \\ \\[-1.1em]
LAP-ResNet & 82.1, [79.5, 84.7] &  \textbf{84.6}, [81.5, 87.6] \\ \bottomrule
\end{tabular}
\caption{Average test accuracy and 95$\%$ confidence interval over 100 instances on the circles 2D dataset with 1000 points and 1 block (in $\%$)} 
\label{tab:1000p1b}
\end{table}

\begin{table}[H]
\centering
\hspace*{-0.8cm} 
\begin{tabular}{ccc}
\toprule
& No batch normalization & Batch normalization  \\ \midrule \\ [-1.1em]
ResNet & 89.1, [87.2, 91.00] & 99.4, [99.0, 99.8] \\ \\[-1.1em]
Regularized ResNet, $\lambda = 0.09$ & 69.7, [65.6, 73.7] &99.5, [98.9, 1.00] \\ \\[-1.1em]
LAP-ResNet & 75.7, [72.8, 78.6] & \textbf{99.8}, [99.7, 1.00] \\ \bottomrule
\end{tabular}
\caption{Average test accuracy and 95$\%$ confidence interval over 100 instances on the circles 2D dataset with 1000 points and 100 blocks (in $\%$)} 
\label{tab:1000p100b}
\end{table}

\begin{table}[H]
\centering
\hspace*{-0.8cm} 
\begin{tabular}{ccc}
\toprule
& No batch normalization & Batch normalization  \\ \midrule \\ [-1.1em]
ResNet & 90.2, [88.8, 91.5]  &  98.0, [97.2, 98.8] \\ \\[-1.1em]
Regularized ResNet, $\lambda = 0.04$ & 89.7, [88.2, 91.3] & \textbf{99.7}, [99.5, 99.9] \\ \\[-1.1em]
LAP-ResNet & 79.1, [75.3, 83.0] & 99.4, [99.0, 99.8]    \\ \bottomrule
\end{tabular}
\caption{Average test accuracy and 95$\%$ confidence interval over 100 instances on the circles 2D dataset with a $\mathcal{N}(0,5)$ initialization (in $\%$)} 
\label{tab:biginit}
\end{table}

\begin{table}[H]
\centering
\hspace*{-0.8cm} 
\begin{tabular}{ccc}
\toprule
& No batch normalization & Batch normalization  \\ \midrule \\ [-1.1em]
ResNet & 88.2, [85.5, 90.1] & 92.9, [90.9, 94.9] \\ \\[-1.1em]
Regularized ResNet, $\lambda = 0.04$ & 93.5, [91.4, 95.6] & 94.4, [92.4, 96.3] \\ \\[-1.1em]
LAP-ResNet &  95.8, [94.0, 97.6] &  \textbf{96.0}, [94.6, 97.3] \\ \bottomrule
\end{tabular}
\caption{Average test accuracy and 95$\%$ confidence interval over 100 instances on the circles 2D dataset with 50 points and 9 blocks (in $\%$)} 
\label{tab:50p9b}
\end{table}

\end{document}